%% file: ITW_main.tex
\documentclass[conference]{IEEEtran}
\IEEEoverridecommandlockouts

\usepackage{cite}
\usepackage{amsmath,amssymb,amsfonts}
\usepackage{amsthm}
\usepackage{algorithmic}
\usepackage{graphicx}
\usepackage{textcomp}
\usepackage{xcolor}
\usepackage{xspace}
\newtheorem{theorem}{Theorem}
\newtheorem{remark}{Remark}
\newtheorem{proposition}{Proposition}

\newtheorem{assumption}{Assumption}
\newtheorem{lemma}{Lemma}
\usepackage[ruled,vlined]{algorithm2e}
\usepackage{booktabs}
\usepackage{graphicx}
\usepackage{subcaption}
\usepackage{natbib}
\newcommand{\kexplorer}{\textsc{$\kappa$-Explorer}\xspace}
\def\BibTeX{{\rm B\kern-.05em{\sc i\kern-.025em b}\kern-.08em
    T\kern-.1667em\lower.7ex\hbox{E}\kern-.125emX}}

\begin{document}

\title{\kexplorer: Minimax Active Model Estimation in MDPs\\

\thanks{ONR, NSF TILOS AI Institute, the UCSD Centers for Machine intelligence, computing, and security (MICS).}
}

\author{\IEEEauthorblockN{Xihe Gu}
\IEEEauthorblockA{\textit{Electrical and Computer Engineering} \\
\textit{University of California, San Diego}\\
La Jolla, USA \\
x9gu@ucsd.edu}
\and
\IEEEauthorblockN{Urbashi Mitra}
\IEEEauthorblockA{\textit{Electrical and Computer Engineering} \\
\textit{University of Southern California}\\
Los Angeles, USA \\
ubli@usc.edu}
\and
\IEEEauthorblockN{Tara Javidi}
\IEEEauthorblockA{\textit{Electrical and Computer Engineering} \\
\textit{University of California, San Diego}\\
La Jolla, USA \\
tjavidi@ucsd.edu}
}

\maketitle

\begin{abstract}
Accurately estimating the transition dynamics of an unknown Markov decision process (MDP) is a central problem
in reinforcement learning. 
The accuracy of overall model estimation, however, depends on both 
the stochasticity of the transitions given state--action pairs, what we call the 
\emph{intrinsic} (statistical) complexity of each pair, and the visitation frequency of each state--action pair, dictated \emph{extrinsically} by the exploration policy and its induced occupancy measure.

We propose \kexplorer, an active exploration algorithm that constructs an exploration reward for each state--action pair, using a $\kappa$-powered ratio of estimated intrinsic complexity to extrinsic visitation frequency of each pair.
By planning with these rewards, \kexplorer adapts its extrinsic visitation frequency to the intrinsic complexity of transitions. 
More specifically, we show that \kexplorer optimizes a surrogate objective function which bounds the minimax estimation error.
Furthermore, we establish a surrogate objective gap of $\widetilde{\mathcal O}(t^{-1/3})$, yielding a finite-sample guarantee for sufficiently large $\kappa$.
Finally, we extend the framework to alternative estimation losses and objectives beyond the minimax setting.
\end{abstract}


\section{Introduction}

Accurately estimating the transition dynamics of an unknown Markov
decision process (MDP) is a central problem in learning-enabled control such as reward-free reinforcement
learning and model-based exploration
\cite{araya2011active,tarbouriech2020active}. 
In tabular MDPs, this boils down to estimating the distribution of transitions to future states given the current state--action pair. However, since these transitions can
exhibit substantially different degree of randomness, the intrinsic estimation
complexity can be substantially different across state-action pairs. For instance, nearly deterministic transitions can be estimated with
relatively few samples, whereas highly unpredictable 
transitions to a large number of possible futrue states require
many more observations. This means that model estimation inherently requires   
uniformly accurate estimates across all state--action pairs. In other words, the quality of the learned model rests upon minimizing the error of the most poorly-estimated transition distribution
\cite{shekhar2020adaptive,shekhar2021adaptive,tarbouriech2020active,
al2023active,halim2025fairness}. 

The problem of minimax  distribution
estimation in the bandit setting has been studied in \cite{shekhar2020adaptive}, where the
learner is tasked to minimize the worst case estimation error across a number of unknown distributions. In the bandit setting, the learner has direct access to sample from any of the distributions whose empirical distributions is used to  derive an upper-confidence-bound (UCB) and an adaptive sampling strategy. 
In an MDP,
however, such a direct (adaptive) sampling is impossible: the learner cannot jump to
the most difficult or under-sampled state--action pair directly; instead the visitation
frequencies must be regulated through trajectories induced by a
policy. This turns complexity-aware allocation into a constrained
exploration problem coupled through the dynamics.

The minimax estimation problem in MDPs is difficult for two distinct reasons. First, the feasible visitation frequencies are coupled through the MDP dynamics, so improving coverage of any one state--action pair may necessarily affect visitation frequencies to other parts of the state space. To address this issue, prior work often considers the problem of optimizing coverage as an important algorithm design objective. In other words, the algorithms are designed to ensure broad coverage by optimizing state-visitation entropy \cite{hazan2019provably}, trajectory entropy \cite{tiapkin2023fast}, or other general concave objective functions~\cite{gu2026relative} that can unify entropy-based objectives. While ensuring broad coverage often ensures visitation frequency is matched to a given target distributions, such as uniform, their applicability to our minmax estimation problem cannot be guaranteed, especially in cases where the instance-dependent optimal exploration is far from the given target (non-uniform visitation). The second challenge for minimax MDP model exploration and estimation is the non-smoothness of the minimax objective. This means that the methods to minimize average exploration error or other smooth objectives~\cite{tarbouriech2020active} cannot be readily applied. 

In this paper, we propose \kexplorer for minimax transition-model estimation by adapting the objective of the framework introduced in \cite{gu2026relative}, and target coverage/occupancy measure, as a function of empirical transition distributions. Our central contribution is yielding a near-minimax approximation guarantee controlled by $\kappa$ by utilizing the notion of intrinsic statistical complexity and the extrinsic visitation frequency
. Furthermore, we establish a finite-sample surrogate objective gap of $\widetilde{\mathcal O}(t^{-1/3})$, matching the convergence order of prior algorithms designed for optimizing smooth occupancy-based exploration objectives such as expected average or log-sum-exp function of estimation error \cite{tarbouriech2019active,tarbouriech2020active,gu2026relative}.
Together, these results show that the empirical allocation is near-optimal for the minimax transition-model estimation objective when $\kappa$ and the exploration budget are sufficiently large.
 Finally, we characterize how parameter and weight choices connect coverage-matching, average-case, and worst-case exploration.

\section{Problem Formulation}
\label{sec:problem_formulation}

We consider a finite controlled Markov chain with state space
$\mathcal S$, $|\mathcal S|=S$, and action space
$\mathcal A$, $|\mathcal A|=A$.
At each time $t$, the agent observes $s_t\in\mathcal S$,
selects $a_t\in\mathcal A$, and transitions to $s_{t+1}$
according to an unknown kernel $P(\cdot\mid s_t,a_t)$.

A stationary randomized policy
$\pi:\mathcal S\to\Delta(\mathcal A)$ specifies the action
distribution at each state; $\Pi$ denotes the set of such policies.
In a tabular setting, the finiteness of state--action pairs, implies that every stationary randomized policy induces at least one invariant occupancy
measure. For a transition kernel $p$, define the set of all stationary occupancy measures 
\begin{align}
\label{eq:D_def}
\mathcal D^{(p)}
:={}&
\Biggl\{
\mathbf d\in\Delta(\mathcal S\times\mathcal A), \pi \in \Pi : 
d_{s,a} = \pi(a\mid s) \sum_{a'\in\mathcal A}d_{s,a'}, 
\notag\\
&\sum_{a\in\mathcal A}d_{s,a}
=
\sum_{(s',a')\in\mathcal S\times\mathcal A}
p(s\mid s',a')d_{s',a'},
\quad \forall s\in\mathcal S
\Biggr\}.
\end{align}
Often we restrict our attention to strictly ergodic policies under which the occupancy metric is unique. In particular, consider  $\eta\in\bigl(0,\frac{1}{2SA}\bigr)$ and transition
kernel $p$, define the $\eta-$ergodic occupancy domain
\begin{equation}
\label{eq:constrainedD_def}
\mathcal D_\eta^{(p)}
:=
\left\{
\mathbf d\in\mathcal D^{(p)}:
d_{s,a}\geq 2\eta,
\quad
\forall(s,a)\in\mathcal S\times\mathcal A
\right\}.
\end{equation}
This restriction ensures that no 
state--action pair is left unexplored. 
Furthermore, every $\mathbf d\in\mathcal D_\eta^{(P)}$ is induced by an ergodic policy
$\pi(a\mid s)=d_{s,a}/\sum_{b\in\mathcal A}d_{s,b}$.
Thus, our optimization of policies can be redunced to optimizing objectives over stationary occupancy objectives 
$\mathcal D_\eta^{(P)}$.

To estimate the transition probability 
for each $(s,a,s')\in\mathcal S\times\mathcal A\times\mathcal S$, we 
define the transition count before time $t$ as
\[
T_{s,a,s'}(t)
:=
\left|
\left\{
\tau<t:
s_\tau=s,\,
a_\tau=a,\,
s_{\tau+1}=s'
\right\}
\right|,
\]
and the state--action visit count as
\[
T_{s,a}(t)
:=
\sum_{s'\in\mathcal S}T_{s,a,s'}(t).
\]
Whenever $T_{s,a}(t)\geq1$, the empirical transition estimator is
\begin{equation}
\label{eq:p_hat_def}
\widehat P_t(s'\mid s,a)
=
\frac{T_{s,a,s'}(t)}{T_{s,a}(t)}.
\end{equation}
If $T_{s,a}(t)=0$, $\widehat P_t(\cdot\mid s,a)$ may be initialized
arbitrarily; throughout the paper, we use the uniform distribution for simplification.

Our goal is to estimate the entire transition kernel over a horizon $t$.
As a stationary-policy benchmark, we minimize the worst expected
transition-estimation error:
\begin{equation}
\label{eq:minimax_error}
\pi^\star
\in
\arg\min_{\pi\in\Pi}
\max_{(s,a)\in\mathcal S\times\mathcal A}
\mathbb E_\pi\!\left[
\left\|
P(\cdot\mid s,a)
-
\widehat P_t(\cdot\mid s,a)
\right\|_1
\right],
\end{equation}
where the expectation is over the trajectory generated by $\pi$
from a fixed initial distribution.

For a fixed positive sample count $T_{s,a}(t)$ and independent
draws from $P(\cdot\mid s,a)$, the empirical estimator satisfies
\begin{equation}
\label{eq:per_transition_error}
\mathbb E\!\left[
\left\|
\widehat P_t(\cdot\mid s,a)
-
P(\cdot\mid s,a)
\right\|_1
\right]
\leq
\frac{c_{s,a}}{\sqrt{T_{s,a}(t)}},
\end{equation}
where for each $(s,a) 
$\begin{equation}
\label{eq:c_def}
c_{s,a}
:=
\sum_{s'\in\mathcal S}
\sqrt{
P(s'\mid s,a)
\bigl(1-P(s'\mid s,a)\bigr)
}.
\end{equation}
measures the \emph{intrinsic (statistical) complexity} associated with estimating the probability distribution of next state.  Throughout the theoretical analysis, we assume $c_{s,a}>0$
for all $(s,a)$.

This observation motivates us to consider
the $\ell_1$ error scaling $\frac{c_{s,a}}{\sqrt{T_{s,a}(t)}}$ for each 
state--action pair. This can be written 
in terms of intrinsic complexity as well as the empirical visitation frequency:
\begin{equation}
\label{eq:empirical_occupancy}
\widehat d_{s,a}(t)
:=
\frac{T_{s,a}(t)}{t},
\end{equation}
externally controlled by the choice of policy.
In other words, the estimation error at each state--action pair converges to zero with a scaling that is
characterized by the ratio between the intrinsic complexity and extrinsic visitation frequency of each state--action pair:
:
\[
\frac{c_{s,a}}{\sqrt{T_{s,a}(t)}}
=
\frac{1}{\sqrt t}
\frac{c_{s,a}}{\sqrt{\widehat d_{s,a}(t)}}.
\]

This suggests a surrogate worst-case normalized performance metric for any induced occupancy measure $\mathbf{d}$ 
\begin{equation}
\label{eq:V_def}
V(\mathbf d)
:=
\max_{(s,a)\in\mathcal S\times\mathcal A}
\frac{c_{s,a}}{\sqrt{d_{s,a}}},
\end{equation}
interpreting $c_{s,a}/\sqrt{d_{s,a}}$ as $+\infty$
when $d_{s,a}=0$.
The minimax occupancy benchmark is
\begin{equation}
\label{eq:Vstar_def}
V^\star
:=
\min_{\mathbf d\in\mathcal D^{(P)}}
V(\mathbf d)
=
\min_{\mathbf d\in\mathcal D^{(P)}}
\max_{s,a}
\frac{c_{s,a}}{\sqrt{d_{s,a}}},
\end{equation}
with a minimax-optimal allocation
\begin{equation}
\label{eq:dstar_def}
\mathbf d^\star
\in
\arg\min_{\mathbf d\in\mathcal D^{(P)}}V(\mathbf d).
\end{equation}



Next we design an adaptive exploration strategy that
allocates extrinsic visitation according to intrinsic transition
complexity, making $V(\widehat{\mathbf d}(n))$ close to $V^\star$
despite the unknown $P$ and $\{c_{s,a}\}$.

\section{\kexplorer: Algorithm and Main Result}
\label{sec:kexplorer}


\subsection{\kexplorer}

The algorithm takes as input a hyperparameter $\kappa$, an occupancy lower-bound parameter $\eta$, a confidence level $\delta$ used to construct uncertainty sets, and a sequence of episode lengths $\{\tau_m\}_{m=1}^M$.
Throughout the exploration process, the algorithm maintains three key quantities: an estimate $\widetilde c_{s,a}$ of the intrinsic complexity, an estimate $\widetilde d_{s,a}$ of the extrinsic visitation frequency, and a confidence neighborhood $\widetilde{\mathcal B}_{t_m}$ of the transition kernel.
After $M-1$ episodes, the algorithm returns the empirical
transition kernel estimate $\widehat P_{t_{M}}$.

\begin{algorithm}[h]
\caption{\kexplorer}
\label{alg:kexplorer_simple}
\begin{algorithmic}[1]
\REQUIRE $\kappa\in\mathbb Z_{>0}$,
constrained parameter $\eta$, confidence level $\delta$,
episode lengths
$\{\tau_m\}_{m=1}^{M-1}$

\FOR{$m=1,\ldots,M-1$}
    \STATE Let $t_m$ be the current time
    \STATE Construct estimates
    $\widetilde c_{s,a}(t_m)$,
    $\widetilde d_{s,a}(t_m)$, and
    $\widetilde{\mathcal B}_{t_m}$
    \STATE Define
    \(
    r_{s,a}(t_m)
    \) using \eqref{eq:generic_reward}
    \STATE Compute the occupancy direction
    \(
    \boldsymbol{\psi}_m
    \)
    using \eqref{eq:generic_direction}
    \STATE Obtain policy 
    \(
    \pi_m(a\mid s)
    \) 
    using \eqref{eq:pim_def}
    \STATE Execute $\pi_m$ for $\tau_m$ steps
\ENDFOR
\RETURN $\widehat P_{t_{M}}$
\end{algorithmic}
\end{algorithm}

Algorithm~\ref{alg:kexplorer_simple} summarizes the procedure.
Let $t_m$ denote the starting time of episode $m$.
At the beginning of each episode, \kexplorer updates the intrinsic transition complexity estimate $\widetilde c_{s,a}(t_m)$
and the extrinsic visitation estimate $\widetilde d_{s,a}(t_m)$ for each state--action pair $(s,a)$.
Based on these estimates, \kexplorer assigns the exploration reward
\begin{equation}
\label{eq:generic_reward}
r_{s,a}(t_m)
:=
\left(
\frac{
\widetilde c_{s,a}(t_m)
}{
\sqrt{\widetilde d_{s,a}(t_m)}
}
\right)^\kappa,
\qquad
\kappa\in\mathbb Z_{>0}.
\end{equation}
Larger values of $\kappa$ place increasing emphasis on state--action
pairs with large complexity-to-visitation ratios.

Next, \kexplorer updates the confidence set $\widetilde{\mathcal B}_{t_m}$ for the transition kernel.
It then computes an exploration direction by jointly optimizing over plausible transition kernels and their corresponding feasible occupancy measures:
\begin{equation}
\label{eq:generic_direction}
\boldsymbol{\psi}_m
\in
\arg\max_{\substack{
\widetilde p\in\widetilde{\mathcal B}_{t_m}\\
\mathbf d\in\mathcal D_\eta^{(\widetilde p)}
}}
\left\langle
 r(t_m),
\mathbf d
\right\rangle.
\end{equation}
The corresponding policy is obtained by normalizing the occupancy
at each state:
\begin{equation}
\label{eq:pim_def}
    \pi_m(a\mid s)
=
\frac{\psi_m(s,a)}
{\sum_{b\in\mathcal A}\psi_m(s,b)}.
\end{equation}

The estimates and episode lengths used in our theoretical analysis are specified below.


\subsection{Main Theorem}

We now specialize \kexplorer to the estimators used in our analysis.
For the intrinsic complexity, we use the optimistic estimate
\begin{equation}
\label{eq:c_ucb}
\widehat c^+_{s,a}(t)
:=
\min
\left\{
\sqrt{S-1},
\widehat c_{s,a}(t)
+
e_{s,a}(t)
\right\},
\end{equation}
where $\widehat c_{s,a}(t):=
\sum_{s'\in\mathcal S}
\sqrt{
\widehat P(s'\mid s,a)
\bigl(1-\widehat P(s'\mid s,a)\bigr)
}$ is the plug-in estimate of $c_{s,a}$ and
$e_{s,a}(t)$ is a confidence radius.

For visitation, define
\[
T^+_{s,a}(t)
:=
\max\{T_{s,a}(t),\epsilon\}.
\]


For each state--action pair $(s,a)$, we define an $\ell_1$ confidence ball around the empirical transition distribution:
\begin{align}
\label{eq:Pball_confset}
\mathcal B_{s,a}(t)
:=
\Bigl\{
&\widetilde P(\cdot\mid s,a)\in\Delta(\mathcal S):
\notag\\
&\left\|
\widetilde P(\cdot\mid s,a)
-
\widehat P_t(\cdot\mid s,a)
\right\|_1
\le b_{s,a}(t)
\Bigr\}.
\end{align}
where $b_{s,a}(t)$ is the corresponding confidence radius.
For convenience, define
\[
\mathcal B_t
:=
\prod_{(s,a)\in\mathcal S\times\mathcal A}
\mathcal B_{s,a}(t).
\]





For the theoretical analysis, we impose the following standard assumption.

\begin{assumption}
\label{ass:communicating-mdp}
The tabular MDP $M=(\mathcal S,\mathcal A,P)$ is communicating. That is,
for every pair of states $s,s'\in\mathcal S$, there exists a policy $\pi$
and a finite time $t<\infty$ such that
\[
\mathbb P_\pi(S_t=s'\mid S_0=s)>0.
\]
\end{assumption}

The communicating assumption yields the following
positivity guarantee for the minimax optimizers.

\begin{lemma}
\label{lem:positive-minimax-occupancy}
Under Assumption~\ref{ass:communicating-mdp}, there exists
$d_{\min}>0$ such that every minimizer $\mathbf d^\star$ in
\eqref{eq:dstar_def} satisfies
\[
\min_{s,a}d^\star_{s,a}\ge d_{\min}.
\]
\end{lemma}

The proof of Lemma~\ref{lem:positive-minimax-occupancy}
is provided in Appendix~\ref{appendix:proof_positive_minimax}.

\begin{theorem}
\label{thm:near-minimax-transition-estimation}
Suppose Assumption~\ref{ass:communicating-mdp} holds,
and let $d_{\min}>0$ be as in
Lemma~\ref{lem:positive-minimax-occupancy}, $\eta$ chosen such that $ \eta \le \frac{1}{2}d_{\min}$.
For any $\varepsilon\in(0,1)$, define
\[
\kappa_0
:=
\max\left\{
3,\,
\left\lceil
\frac{\log(1/d_{\min})}{\log(1+\varepsilon)}
\right\rceil
\right\}.
\]
Run \kexplorer with
\[
\widetilde c_{s,a}(t)=\widehat c^+_{s,a}(t),
\quad
\widetilde d_{s,a}(t)=\frac{T^+_{s,a}(t)}{t},
\quad
\widetilde{\mathcal B}_{t_m}=\mathcal B_{t_m},
\]
an appropriate episode schedule, and any integer $\kappa\ge\kappa_0$.
Then, with high probability, after $M-1$ episodes
\[
V\!\left(\widehat{\mathbf d}(t_M)\right)
\le
\frac{1+\varepsilon}{1-\varepsilon}V^\star
+
\widetilde{\mathcal O}\!\left(t_M^{-1/3}\right).
\]

\end{theorem}

The analysis supporting this guarantee is developed in Section~\ref{sec:proof_main}.

\section{Theoretical Analysis}
\label{sec:proof_main}

We first establish the connection between $U_\kappa$ and the minimax criterion, then analyze the convergence of \kexplorer.

\subsection{Surrogate Objective and Minimax Approximation}

The parameter \(\kappa\) controls how strongly \kexplorer prioritizes large exploration rewards. Since increasing \(\kappa\) amplifies differences among state--action pairs and increasingly concentrates the planning direction on those with the largest complexity-to-visitation ratios. In the large-\(\kappa\) regime, this behavior is therefore naturally connected to the minimax criterion in \eqref{eq:V_def}.

For $\kappa\in\mathbb Z_{>0}$, define the occupancy objective $U_\kappa:(0,1)^{S\times A}\to\mathbb R$ by
\begin{equation}
\label{eq:Ukappa_definition}
U_\kappa(\mathbf d)
=
\begin{cases}
\displaystyle
\sum_{s,a} c_{s,a}^2\log d_{s,a},
& \kappa=2, \\[1em]
\displaystyle
\sum_{s,a}
\frac{c_{s,a}^\kappa}
{1-\frac{\kappa}{2}}
d_{s,a}^{\,1-\frac{\kappa}{2}},
& \kappa\neq2.
\end{cases}
\end{equation}

The objective $U_\kappa$ is concave and decomposable across state--action
pairs. Its gradient satisfies
\begin{equation}
\label{eq:Ukappa_gradient_identity}
\bigl(\nabla U_\kappa(\mathbf d)\bigr)_{s,a}
=
\left(
\frac{c_{s,a}}
{\sqrt{d_{s,a}}}
\right)^\kappa.
\end{equation}
The gradient identity identifies the exploration reward with $\nabla U_\kappa$, providing a Frank--Wolfe interpretation of the planning step.

The same objective also becomes increasingly aligned with the minimax
criterion as $\kappa$ grows.

\begin{proposition}
\label{proposition:kappa_inf_l1}
For any $\mathbf d \in \mathcal{D}^{(P)}$, assume $d_{s,a}>0$ for all $(s,a)$, and let
\[
d_{\min}
:=
\min_{s,a}d_{s,a}.
\]
For any $\varepsilon\in(0,1)$, if
\[
\kappa
\ge
\frac{
\log(1/d_{\min})
}{
\log(1+\varepsilon)
},
\]
then
\begin{equation}
\label{eq:Uk_minimax_epsilon_l1}
(1-\varepsilon)V(\mathbf d)
\le
\left[
\left(1-\frac{\kappa}{2}\right)
U_\kappa(\mathbf d)
\right]^{1/\kappa}
\le
(1+\varepsilon)V(\mathbf d).
\end{equation}
\end{proposition}

The proof is provided in
Appendix~\ref{appendix:Uk_minimax_proof}.
Proposition~\ref{proposition:kappa_inf_l1} shows that, for sufficiently large
$\kappa$, the objective optimized by \kexplorer becomes arbitrarily close,
after the above monotone transformation, to the minimax criterion $V$.

Let
\[
\mathbf d_\kappa^\star
\in
\arg\max_{\mathbf d\in\mathcal D^{(P)}}
U_\kappa(\mathbf d)
\]
denote an optimizer of the $\kappa$-parameterized objective.
We next show that the empirical occupancy generated \kexplorer converges to this optimizer.

\subsection{Convergence Analysis}

We first discuss the choices of $\widetilde c_{s,a}$ and $\widetilde{\mathcal B}_{t_m}$ in
Theorem~\ref{thm:near-minimax-transition-estimation}.
These quantities provide the confidence bounds on the intrinsic transition complexity and the unknown transition kernel needed for the optimistic planning step.

\begin{lemma}
\label{lem:uniform_conf_P_c}
Fix $\delta\in(0,1)$.
Let $\widehat c^+_{s,a}(t)$ and $\mathcal B_{s,a}(t)$ be defined in \eqref{eq:c_ucb} and \eqref{eq:Pball_confset}, respectively, with the confidence radii and conventions for $T_{s,a}(t)=0$ specified in Appendix~\ref{appendix:conf_radii}.

Then, with probability at least $1-\delta$, simultaneously for all $(s,a)\in\mathcal S\times\mathcal A$ and all $t\ge1$,
\[
P(\cdot\mid s,a)\in\mathcal B_{s,a}(t),
\qquad
c_{s,a}\le\widehat c^+_{s,a}(t).
\]
\end{lemma}

Secondly, the restricted occupancy domain $\mathcal D_\eta^{(\widetilde p)}$
keeps every state--action occupancy bounded away from zero, ensuring
that the gradient of $U_\kappa$ is Lipschitz continuous over the domain.

With these prerequisites in place, we now show that \kexplorer converges to an optimizer of $U_\kappa$.

\begin{theorem}
\label{thm:general_convergence_rate}
Fix $\kappa\ge1$ and let the episode lengths satisfy
\[
\tau_m=\tau_1m^2,
\]
so that
\[
t_m
=
1+
\tau_1\frac{(m-1)m(2m-1)}{6}.
\]
Assume that $\eta$ is chosen such that
\(
2\eta
\le
\min_{s,a}d_{\kappa,s,a}^\star.
\)
Then, with high probability, after $M-1$ episodes
\[
U_\kappa(\mathbf d_\kappa^\star)
-
U_\kappa\!\left(\widehat{\mathbf d}(t_M)\right)
=
\widetilde{\mathcal O}\!\left(t_M^{-1/3}\right).
\]
\end{theorem}

The proof follows a Frank--Wolfe-style inductive argument. The
one-step suboptimality recursion consists of a contraction term, a
quadratic step-size term, a tracking error between the stationary
occupancy and its empirical realization, and an optimization error due
to optimistic planning under unknown dynamics. These terms are controlled
using empirical-occupancy concentration and the confidence bounds above.
The full proof is provided in
Appendix~\ref{appendix: proof_thm_general_convergence_rate}.

\begin{remark}
The constants in the finite-sample convergence bound depend on both the occupancy threshold $\eta$ and the mixing properties of the induced Markov chains.
Larger values of $\eta$ improve the smoothness constants, but may exclude $\mathbf d_\kappa^\star$ from
$\mathcal D_\eta^{(P)}$.
Under Assumption~\ref{ass:communicating-mdp}, an argument analogous to the proof of Lemma~\ref{lem:positive-minimax-occupancy} shows that, for $\kappa\ge2$, an optimizer $\mathbf d_\kappa^\star$ exists and satisfies
$\min_{s,a}d_{\kappa,s,a}^\star>0$.
Thus, choosing $0<\eta\le\frac12\min_{s,a}d_{\kappa,s,a}^\star$
preserves the optimizer.
Smaller values of $\eta$ or slower mixing, as captured by $\gamma_{\min}$, can increase the constants in the bound, while the asymptotic rate remains $\widetilde{\mathcal O}(t_M^{-1/3})$ for fixed $\kappa$.
\end{remark}



\begin{proof}[Proof sketch of
Theorem~\ref{thm:near-minimax-transition-estimation}]
By Theorem~\ref{thm:general_convergence_rate},
the objective gap is bounded above by
$\widetilde{\mathcal O}(t_M^{-1/3})$
with high probability.
For $\kappa>2$, optimality of $\mathbf d_\kappa^\star$
and concavity of $x\mapsto x^{1/\kappa}$ imply

\vspace{-0.1in}
{
\small
\[
\left[
\left(1-\frac{\kappa}{2}\right)
U_\kappa(\widehat{\mathbf d}(t_M))
\right]^{1/\kappa}
\le
\left[
\left(1-\frac{\kappa}{2}\right)
U_\kappa(\mathbf d^\star)
\right]^{1/\kappa}
+\widetilde{\mathcal O}(t_M^{-1/3}).
\]
}

For all sufficiently large $M$, the empirical occupancy
satisfies
$\min_{s,a}\widehat d_{s,a}(t_M)\ge\eta\ge d_{\min}$.
Proposition~\ref{proposition:kappa_inf_l1}
therefore yields
\[
(1-\varepsilon)V(\widehat{\mathbf d}(t_M))
\le
(1+\varepsilon)V^\star
+\widetilde{\mathcal O}(t_M^{-1/3}),
\]
which proves the claim.
The full proof is provided in
Appendix~\ref{appendix:proof_cor}.
\end{proof}




\section{Discussion and Conclusion}
\label{sec:discussion}

The main development of this paper focuses on minimax transition
estimation under the $\ell_1$ error. We conclude by discussing two broader
perspectives of the framework: the interpretation of particular choices
of $\kappa$, and extensions to other transition-estimation metrics.

\subsection{Specific Choices of $\kappa$}
While the main theorem focuses on sufficiently large $\kappa$, particular finite values also admit useful interpretations.

For $\kappa=2$, the objective becomes
\(
U_2(\mathbf d)
=
\sum_{s,a} c_{s,a}^2 \log d_{s,a}.
\)
Let
\(
\bar c_{s,a}^2
:=
\frac{c_{s,a}^2}
{\sum_{s',a'}c_{s',a'}^2}.
\)
Up to a positive multiplicative constant,
\[
U_2(\mathbf d)
\propto
\sum_{s,a}\bar c_{s,a}^2\log d_{s,a}
=
-\mathrm{KL}(\bar{\mathbf c}^{\,2}\Vert\mathbf d)
-
H(\bar{\mathbf c}^{\,2}).
\]
Since the entropy term is independent of $\mathbf d$, maximizing
$U_2$ is equivalent to minimizing
$\mathrm{KL}(\bar{\mathbf c}^{\,2}\Vert\mathbf d)$.
Thus, $\kappa=2$ admits a coverage-matching interpretation related to
State Marginal Matching~\cite{lee2019efficient}.

The framework can also accommodate alternative choices of intrinsic
complexity weights. For example, define
\[
c^{\mathrm{avg}}_{s,a}
:=
\left(
\sum_{s'\in\mathcal S}
\sqrt{
P(s'\mid s,a)
\bigl(1-P(s'\mid s,a)\bigr)
}
\right)^{1/3}.
\]
Using these weights together with $\kappa=3$ gives
\[
U_3(\mathbf d)
=
-2
\sum_{s,a}
\frac{
\sum_{s'\in\mathcal S}
\sqrt{
P(s'\mid s,a)
\bigl(1-P(s'\mid s,a)\bigr)
}
}{
\sqrt{d_{s,a}}
},
\]
which, up to scaling and sign, corresponds to the average-case
transition-estimation objective studied in
\cite{tarbouriech2020active}. Importantly, this connection involves both
the choice of $\kappa$ and an alternative intrinsic complexity weight,
rather than a change in $\kappa$ alone.

These examples suggest a broader interpretation of the framework:
$\kappa$ controls the curvature, and hence how strongly the exploration
objective emphasizes bottleneck state--action pairs, while the choice of
intrinsic complexity weights determines the statistical notion of local
estimation difficulty.

\subsection{Other Distance Measures}

Although we focus on the $\ell_1$ transition-estimation error, the same
exploration principle can be applied to other distance measures by
starting from their corresponding local estimation-error behavior.

The $\ell_1$ criterion in \eqref{eq:per_transition_error} scales as $T_{s,a}^{-1/2}$. Other distance measures suggest different intrinsic complexity weights and visitation penalties.

For the squared $\ell_2$ distance, the local estimation error instead
satisfies
\[
\mathbb E\!\left[
\left.
\left\|
P(\cdot\mid s,a)
-
\widehat P(\cdot\mid s,a)
\right\|_2^2
\,\right|\,
T_{s,a}
\right]
=
\frac{c^{\ell_2}_{s,a}}{T_{s,a}},
\]
where
\(
c^{\ell_2}_{s,a}
:=
\sum_{s'\in\mathcal S}
P(s'\mid s,a)
\bigl(1-P(s'\mid s,a)\bigr).
\)
Thus, the corresponding exploration signal would naturally depend on
$c^{\ell_2}_{s,a}/d_{s,a}$ rather than
$c_{s,a}/\sqrt{d_{s,a}}$.

A similar phenomenon arises for KL divergence. Under appropriate
regularity conditions, its conditional error admits the expansion
\[
\mathbb E\!\left[
\left.
D_{\mathrm{KL}}
\!\left(
\widehat P(\cdot\mid s,a)
\Vert
P(\cdot\mid s,a)
\right)
\,\right|\,
T_{s,a}
\right]
\approx
\frac{S-1}{2T_{s,a}}
+
\frac{c^{\mathrm{KL}}_{s,a}}{T_{s,a}^2},
\]
where
\(
c^{\mathrm{KL}}_{s,a}
:=
\frac{1}{12}
\left(
\sum_{s'\in\mathcal S}
\frac{1}{P(s'\mid s,a)}
-1
\right).
\)
This expansion again suggests a metric-dependent exploration reward,
with both the intrinsic complexity and the dependence on visitation determined by the underlying estimation loss
\cite{shekhar2020adaptive}.

These examples indicate that the central principle of \kexplorer is not
specific to the $\ell_1$ distance: local transition-estimation error
determines an intrinsic complexity and a visitation-dependent sampling
penalty, which can then be used as an exploration reward for planning.
Extending the finite-sample analysis developed here to these alternative
distance measures requires adapting the corresponding global objective
and smoothness analysis, and is an interesting direction for future work.

\subsection{Conclusion}

We studied minimax active transition-model estimation in finite MDPs. Starting from the local estimation error, \kexplorer converts complexity-to-visitation ratios into exploration rewards and plans under the occupancy constraints imposed by the MDP dynamics. We showed that
these rewards correspond to the gradient of a smooth objective $U_\kappa$, that \kexplorer converges to its optimizer at rate
$\widetilde{\mathcal O}(t^{-1/3})$, and that sufficiently large $\kappa$ yields a near-minimax allocation.

Looking forward, several directions remain open. One is to extend the framework beyond finite tabular MDPs to high-dimensional or continuous state spaces, where both transition estimation and occupancy optimization become substantially more challenging. Another is to develop a Bayesian interpretation of the intrinsic complexity terms and investigate whether prior information can be incorporated into the exploration rewards.
It is also interesting to apply \kexplorer to reward-free reinforcement learning~\cite{jin2020reward,kaufmann2021adaptive,menard2021fast}, where the estimated transition model can support planning across downstream tasks. Incorporating task-dependent visitation requirements into the exploration objective may further support zero-shot planning~\cite{touati2022does, touati2021learning} for newly specified rewards without additional environment interaction.



\newpage

\section*{AI Statement}
ChatGPT (OpenAI) was used for language editing throughout the manuscript and to assist with drafting and organizing the manuscript outline. The algorithm, theoretical results, and conclusions were developed independently by the authors.

\bibliography{main}

\newpage

\appendices
\input{appendix}

\end{document}

%% file: appendix.tex
\section{Proof of Lemma~\ref{lem:positive-minimax-occupancy}}
\label{appendix:proof_positive_minimax}

\begin{proof}
Consider the stationary policy
$\pi_0(a\mid s)=1/A$.
The communicating assumption implies that the directed graph
with an edge from $s$ to $s'$ whenever
$P(s'\mid s,a)>0$ for some action $a$ is strongly connected.
Since $\pi_0$ assigns positive probability to every action,
its induced transition matrix
\[
P_{\pi_0}(s'\mid s)
=\sum_{a\in\mathcal A}\pi_0(a\mid s)P(s'\mid s,a)
\]
is irreducible. Hence it admits a stationary distribution
$\mu_0$ with $\mu_0(s)>0$ for every $s\in\mathcal S$.
The occupancy measure
\[
d^0_{s,a}:=\mu_0(s)\pi_0(a\mid s)
\]
therefore belongs to $\mathcal D^{(P)}$ and is strictly positive.
In particular,
\[
B:=V(\mathbf d^0)<\infty.
\]

The stationary occupancy polytope $\mathcal D^{(P)}$ is compact,
and the extended-real-valued function $V$ is lower semicontinuous. Thus $V$ attains its minimum on
$\mathcal D^{(P)}$, and its minimum is finite because
$\mathbf d^0$ is feasible.

For any minimizer $\mathbf d^\star$, optimality gives
$V(\mathbf d^\star)\le B$. Consequently, for every $(s,a)$,
\[
\frac{c_{s,a}}{\sqrt{d^\star_{s,a}}}\le B,
\qquad\text{and hence}\qquad
d^\star_{s,a}\ge\frac{c_{s,a}^2}{B^2}.
\]
Since the state and action spaces are finite,
$c_{\min}:=\min_{s,a}c_{s,a}>0$.
Taking
\[
d_{\min}:=\frac{c_{\min}^2}{B^2}>0
\]
proves the claim.
\end{proof}

\section{Proof of Theorem~\ref{thm:general_convergence_rate}}
\label{appendix: proof_thm_general_convergence_rate}

\subsection{Preliminaries: Smoothed Optimization}

Recall the definition of $\mathcal{D}_{\eta }^{(p)}$ from
\eqref{eq:constrainedD_def} that for any
$\eta \in(0,\frac{1}{2SA})$, the constrained set
\[
\mathcal{D}_{\eta }^{(p)}
:=
\Bigl\{
\mathbf d \in \mathcal{D}^{(p)} :
d_{s,a} \ge 2\eta ,
\ \forall (s,a)\in\mathcal{S}\times\mathcal{A}
\Bigr\}.
\]

Recall that the $\kappa$-parameterized objective is
\[
U_\kappa(\mathbf d)
=
\frac{1}{1-\kappa/2}
\sum_{s,a}
c_{s,a}^{\kappa}
d_{s,a}^{1-\kappa/2}.
\]

On the restricted domain, the objective $U_\kappa$ enjoys standard
smoothness properties, as stated in the following lemma.

\begin{lemma}
\label{lem:smooth_Ukappa}
The function $U_\kappa$ is $C_{\eta}$-smooth on
$\mathcal{D}_{\eta }^{(p)}$, i.e.,
\[
\|\nabla U_\kappa(\mathbf d)-\nabla U_\kappa(\mathbf d')\|_2
\le
C_{\eta }\,\|\mathbf d-\mathbf d'\|_2,
\qquad
\forall \mathbf d,\mathbf d' \in \mathcal{D}_{\eta }^{(p)}.
\]
Moreover, the same constant remains valid whenever
$d_{s,a}\ge\eta$ for all $(s,a)$.
We may take
\[
C_{\eta }
=
\frac{\kappa}{2}
c_{\max}^{\kappa}
\eta^{-(1+\kappa/2)},
\qquad
c_{\max} := \max_{s,a} c_{s,a}.
\]
\end{lemma}

\begin{proof}
We compute the Hessian of $U_\kappa$ and bound its operator norm.
Since $U_\kappa$ is separable across coordinates $d_{s,a}$,
its Hessian is diagonal.

For each $(s,a)$,
\begin{align}
\frac{\partial U_\kappa(\mathbf d)}
{\partial d_{s,a}}
&=
\frac{c_{s,a}^{\kappa}}{1-\kappa/2}
\left(1-\frac{\kappa}{2}\right)
d_{s,a}^{-\kappa/2}
\nonumber\\
&=
c_{s,a}^{\kappa}d_{s,a}^{-\kappa/2}
=
\left(
\frac{c_{s,a}}{\sqrt{d_{s,a}}}
\right)^\kappa.
\label{eq:gradient_Ukappa}
\end{align}
Moreover,
\[
\frac{\partial^2 U_\kappa(\mathbf d)}
{\partial d_{s,a}^2}
=
-\frac{\kappa}{2}
c_{s,a}^{\kappa}
d_{s,a}^{-(1+\kappa/2)}.
\]
Thus
\[
\nabla^2 U_\kappa(\mathbf d)
=
\operatorname{diag}
\left(
-\frac{\kappa}{2}
c_{s,a}^{\kappa}
d_{s,a}^{-(1+\kappa/2)}
\right),
\]
and therefore
\[
\|\nabla^2 U_\kappa(\mathbf d)\|_{\mathrm{op}}
=
\max_{s,a}
\frac{\kappa}{2}
c_{s,a}^{\kappa}
d_{s,a}^{-(1+\kappa/2)}.
\]

Whenever $d_{s,a}\ge\eta$ for all $(s,a)$,
\[
\|\nabla^2 U_\kappa(\mathbf d)\|_{\mathrm{op}}
\le
\frac{\kappa}{2}
c_{\max}^{\kappa}
\eta^{-(1+\kappa/2)}
=
C_\eta.
\]
In particular, this bound holds on
$\mathcal D_\eta^{(p)}$, where $d_{s,a}\ge2\eta$.

Finally, for any twice differentiable function,
a uniform bound on the Hessian operator norm implies
Lipschitz continuity of the gradient.
This completes the proof.
\end{proof}

Lemma~\ref{lem:smooth_Ukappa} guarantees that, under a mild interior
constraint, the gradient of $U_\kappa$ is Lipschitz continuous.

Furthermore, according to Lemma 5 in
\cite{tarbouriech2019active}, we have the important property that,
for any $\delta\in(0,1)$, there exists an episode $m_\delta$
such that, for all episodes $m$ succeeding it (and including it),
with probability at least $1-\delta$,
\begin{equation}
\label{eq:d_hat_lbound}
\widehat{d}_{s,a}(t_m)
\ge
\eta,
\qquad
\forall (s,a)\in\mathcal S\times\mathcal A,
\qquad
\forall m\ge m_\delta.
\end{equation}
Notice that Lemma~\ref{lem:smooth_Ukappa} was stated with the slightly
looser constant corresponding to the lower bound $\eta$, precisely so
that the smoothness inequality also applies to the empirical
occupancies satisfying \eqref{eq:d_hat_lbound}.

\subsection{Preliminaries: Bound on Empirical Occupancy}
\begin{lemma}
\label{lma:bound_empirical_occupancy}
Let $\pi$ be a stationary policy inducing an ergodic and reversible chain $P_\pi$ with
spectral gap $\gamma_\pi$ and stationary state occupancy $\psi_\pi(s)$. 
Let the stationary state--action distribution be
$\psi_\pi(s,a) := \psi_\pi(s)\,\pi(a\mid s)$.
We denote by $\gamma_{\min} = \min_{\pi \in \Pi} \gamma_\pi$ the smallest spectral
gap across all policies.
Let $\widehat\psi_n(s,a) := \frac{T_n(s,a)}{n}$ denote the empirical occupancy,
where $T_n(s,a)$ denotes the number of visits to the state--action pair $(s,a)$
during the first $n$ steps when executing policy $\pi$.
Fix $\delta\in(0,1)$. With probability at least $1-\delta$, the following holds for all $(s,a)\in\mathcal S\times\mathcal A$,
\small
\begin{equation}\small
\label{eq:occ_conc}
\big|\widehat\psi_n(s,a)-\psi_\pi(s,a)\big|
\;\le\;
\frac{\sqrt{2B/\gamma_{\min}}+\sqrt{\ln(4SA/\delta)/2}}{\sqrt{n}}
\;+\;
\frac{20B}{\gamma_{\min}n},
\end{equation}
where $S:=|\mathcal S|$, $A:=|\mathcal A|$, and
\begin{equation}\label{eq:B_def}
B \;=\; \log\!\left(\frac{S^2A^{S+1}}{\delta/2}\sqrt{\frac{1}{\eta }}\right).
\end{equation}
\end{lemma}

\begin{proof}
Fix a stationary policy $\pi$.  
Define the empirical state and state--action occupancies as
\[
\widehat\psi_n(s) := \frac{T_n(s)}{n},
\qquad
\widehat\psi_n(s,a) := \frac{T_n(s,a)}{n},
\]
where $T_n(s) = \sum_{a \in \mathcal{A}}T_n(s, a)$.

From Lemma 4 in \cite{tarbouriech2019active}, we have that with probability at least $1-\delta$
\[
\big|\widehat\psi_n(s) - \psi_\pi(s)\big| \le f(n, \delta) = \sqrt{ \frac{2B}{\gamma_{\min} n} }
\;+\;
\frac{20B}{\gamma_{\min}n}, \quad \forall s\in \mathcal{S}.
\]

Using the identity
\[
\widehat\psi_n(s,a) = \widehat\psi_n(s)\,\widehat{\pi}_n(a\mid s),
\]
\[
\widehat{\pi}_n(a\mid s) := \frac{T_n(s,a)}{T_n(s)} \;\; \text{(when $T_n(s)>0$)},
\]
we obtain
\begin{align*}
\big|\widehat\psi_n(s,a) - \psi_\pi(s,a)\big|
&=
\big|\widehat\psi_n(s)\widehat{\pi}_n(a\mid s) - \psi_\pi(s)\pi(a\mid s)\big| \\
&\le
\big|\widehat\psi_n(s) - \psi_\pi(s)\big|\,\pi(a\mid s) \\
&\qquad+
\widehat\psi_n(s)\,\big|\widehat{\pi}_n(a\mid s) - \pi(a\mid s)\big|.
\end{align*}

The first term is controlled by the assumed state occupancy bound.
For the second term, the actions taken at
state $s$ are i.i.d.\ according to $\pi(\cdot\mid s)$, and Hoeffding's
inequality yields that for any $\delta\in(0,1)$,
\[
\Pr\!\left(
\big|\widehat{\pi}_n(a\mid s) - \pi(a\mid s)\big|
\ge
\sqrt{\frac{\ln(2/\delta)}{2T_n(s)}}
\right)
\le \delta.
\]

Combining the bounds, we conclude that with
probability at least $1-2\delta$,
\begin{align}\notag
    \big|\widehat\psi_n(s,a) - \psi_\pi(s,a)\big|
    &\;\le\;
    \pi(a\mid s)\,f(n, \delta)
    +
    \widehat\psi_n(s)\sqrt{\frac{\ln(2/\delta)}{2\,T_n(s)}} \\ \notag
    &\;\le\;
    f(n, \delta)
    +
    \widehat\psi_n(s)\sqrt{\frac{\ln(2/\delta)}{2\,n \widehat\psi_n(s)}} \\ \notag
    &=
    f(n, \delta)
    +
    \sqrt{\frac{\ln(2/\delta)\widehat\psi_n(s)}{2\,n }} \\ \notag
    &\;\le\;
    f(n, \delta)
    +
    \sqrt{\frac{\ln(2/\delta)}{2\,n }}
\end{align}

Hence, for a single state-action pair, with probability at least $1-\delta$,
\begin{align}
\notag
    \big|\widehat\psi_n(s,a) - \psi_\pi(s,a)\big|
    &\;\le\;
    \sqrt{ \frac{2B}{\gamma_{\min} n} }
\;+\;
\frac{20B}{\gamma_{\min}n}
    +
    \sqrt{\frac{\ln(4/\delta)}{2\,n }}\\ \notag
    &= \frac{\sqrt{2B/\gamma_{\min}} + \sqrt{\ln(4/\delta)/2}}{\sqrt{n}} + \frac{20B}{\gamma_{\min}n}
\end{align}
where
$$B = \log\!\left( \frac{SA^{S}}{\delta/2} \sqrt{\frac{1}{\eta }} \right).$$

We need to take a union bound over state action pairs which leads to $\delta = \frac{\delta'}{SA}$ in the high-probability guarantees.

Therefore, with probability at least $1-\delta$,
\begin{align}\notag
    \big|\widehat\psi_n(s,a) - \psi_\pi(s,a)\big| 
    \le \frac{\sqrt{2B/\gamma_{\min}} + \sqrt{\ln(4SA/\delta)/2}}{\sqrt{n}} \\ \notag
    + \frac{20B}{\gamma_{\min}n}, 
    \quad
    \forall (s, a) \in \mathcal{S}\times\mathcal{A}
\end{align}
where
$$B = \log\!\left( \frac{S^2A^{S+1}}{\delta/2} \sqrt{\frac{1}{\eta }} \right).$$
\end{proof}

\subsection{Confidence Radii for
Lemma~\ref{lem:uniform_conf_P_c}}
\label{appendix:conf_radii}

We specify the confidence radii used in Lemma~\ref{lem:uniform_conf_P_c}.
Fix $\delta\in(0,1)$.
For each $(s,a)$ with $T_{s,a}(t)\ge1$, define
\[
\delta_t
:=
\frac{\delta}{
(\pi^2/3)SA\,t^2
}.
\]
\[
e_{s,a}(t)
=
S
\sqrt{
\frac{
2\log(2S/\delta_t)
}{
T_{s,a}(t)
}
}.\]
\[
b_{s,a}(t)
:=
\min\left\{
2,\,
\sqrt{
\frac{
2\bigl(S\log 2+\log(1/\delta_t)\bigr)
}{
T_{s,a}(t)
}
}
\right\}.
\]

For state--action pairs with $T_{s,a}(t)=0$, set
\[
\widehat c^+_{s,a}(t)=\sqrt{S-1},
\qquad
b_{s,a}(t)=2.
\]
The corresponding confidence set is therefore
$\mathcal B_{s,a}(t)=\Delta(\mathcal S)$.

The construction follows the concentration arguments in
\cite{shekhar2020adaptive,tarbouriech2020active}.

\subsection{Core of Proof}

The proof is based on Appendix.D.2 in
\cite{tarbouriech2019active}.
Denote the event described in
Lemma~\ref{lem:uniform_conf_P_c}
by $\mathcal{Q}$.
Specifically, fix $\delta\in(0,1)$.
Let $\mathcal Q$ be the event under which, for all $t\ge1$
and all $(s,a)\in\mathcal S\times\mathcal A$,
\[
P(\cdot\mid s,a)\in \mathcal B_{s,a}(t)
\quad\text{and}\quad
c_{s,a}
\le
\widehat c^+_{s,a}(t).
\]
According to Lemma~\ref{lem:uniform_conf_P_c},
we have $\Pr(\mathcal Q)\ge1-\delta$.
In the following proof, we work on the event $\mathcal Q$.

\begin{proof}

We define the approximation error after episode $m$
(i.e., at time $t_{m+1}-1$) as
\[
\rho_{m+1}
:=
U_\kappa(\mathbf d_\kappa^\star)
-
U_\kappa(\widehat{\mathbf d}_{m+1}),
\]
where $\mathbf d_\kappa^\star$ denotes an optimal solution to the
$\kappa$-parameterized objective, and
$\widehat{\mathbf d}_{m+1}$ is the empirical occupancy after episode
$m$.

Using the occupancy update defined below, this error can be rewritten
as
\[
\rho_{m+1}
=
U_\kappa(\mathbf d_\kappa^\star)
-
U_\kappa
\left(
(1-\beta_m)\widehat{\mathbf d}_m
+
\beta_m\widehat\psi_m
\right).
\]

During episode $m$, let $\nu_m(s,a)$ denote the number of times
action $a$ is taken in state $s$ when executing policy $\pi_m$, and
define the empirical within-episode visitation frequency
\[
\widehat\psi_m(s,a)
:=
\frac{\nu_m(s,a)}{\tau_m}.
\]

The empirical occupancy measure is then updated as a convex
combination
\begin{equation}
\label{eq:occ_update}
\begin{split}
    \widehat{\mathbf d}(t_{m+1})
&=
\frac{\tau_m}{t_{m+1}-1}\widehat\psi_m
+
\frac{t_m-1}{t_{m+1}-1}\widehat{\mathbf d}(t_m) \\
&=
(1-\beta_m)\widehat{\mathbf d}(t_m)
+
\beta_m\widehat\psi_m,
\end{split}
\end{equation}
where
\[
\beta_m
:=
\frac{\tau_m}{t_{m+1}-1}
\]
serves as the step size.
For notational convenience, we write
$\widehat{\mathbf d}_m:=\widehat{\mathbf d}(t_m)$.

Let
\[
\psi^\star_m
=
\arg\max_{\mathbf d\in\mathcal D_\eta^{(P)}}
\left\langle
\nabla U_\kappa(\widehat{\mathbf d}_m),
\mathbf d
\right\rangle
\]
be the state--action stationary distribution that the exact
Frank--Wolfe oracle would return at episode $m$.

Using the $C_\eta$-smoothness of $U_\kappa$, we obtain
\small
\[
\begin{aligned}
\rho_{m+1}
&\le
U_\kappa(\mathbf d_\kappa^\star)
-
U_\kappa(\widehat{\mathbf d}_m)
-
\beta_m
\left\langle
\nabla U_\kappa(\widehat{\mathbf d}_m),
\widehat\psi_m-\widehat{\mathbf d}_m
\right\rangle
+
C_\eta\beta_m^2
\\[0.4em]
&=
U_\kappa(\mathbf d_\kappa^\star)
-
U_\kappa(\widehat{\mathbf d}_m)
-
\beta_m
\left\langle
\nabla U_\kappa(\widehat{\mathbf d}_m),
\psi^\star_m-\widehat{\mathbf d}_m
\right\rangle
+
C_\eta\beta_m^2
\\
&\qquad
-
\beta_m
\left\langle
\nabla U_\kappa(\widehat{\mathbf d}_m),
\widehat\psi_m-\psi^\star_m
\right\rangle
\\[0.4em]
&\le
U_\kappa(\mathbf d_\kappa^\star)
-
U_\kappa(\widehat{\mathbf d}_m)
-
\beta_m
\left\langle
\nabla U_\kappa(\widehat{\mathbf d}_m),
\mathbf d_\kappa^\star-\widehat{\mathbf d}_m
\right\rangle
+
C_\eta\beta_m^2
\\
&\qquad
-
\beta_m
\left\langle
\nabla U_\kappa(\widehat{\mathbf d}_m),
\widehat\psi_m-\psi^\star_m
\right\rangle
\\[0.4em]
&\le
(1-\beta_m)\rho_m
+
C_\eta\beta_m^2
+
\beta_m
\underbrace{
\left\langle
\nabla U_\kappa(\widehat{\mathbf d}_m),
\psi^\star_m-\psi_m
\right\rangle
}_{\epsilon_m} \\[0.4em]
&\qquad
+
\beta_m
\underbrace{
\left\langle
\nabla U_\kappa(\widehat{\mathbf d}_m),
\psi_m-\widehat\psi_m
\right\rangle
}_{\Delta_m}.
\end{aligned}
\]

The first step follows from the $C_\eta$-smoothness of $U_\kappa$.
The second inequality uses the Frank--Wolfe optimality of
$\psi_m^\star$, while the final step follows from the concavity of
$U_\kappa$.

This yields the core recursion
\begin{equation}
\label{eq:algo2_core_recursion}
\rho_{m+1}
\le
(1-\beta_m)\rho_m
+
C_\eta\beta_m^2
+
\beta_m\Delta_m
+
\beta_m\epsilon_m,
\end{equation}
where
\[
\Delta_m
:=
\left\langle
\nabla U_\kappa(\widehat{\mathbf d}_m),
\psi_m-\widehat\psi_m
\right\rangle,
\]
and
\[
\epsilon_m
:=
\left\langle
\nabla U_\kappa(\widehat{\mathbf d}_m),
\psi_m^\star-\psi_m
\right\rangle.
\]

The term $\Delta_m$ captures the tracking error between the stationary
occupancy and its empirical realization, whereas $\epsilon_m$
captures the optimization/estimation error induced by optimistic
planning under unknown transition dynamics.

\paragraph{Bound $\Delta_m$:}

For any $m\ge m_\delta$, by
\eqref{eq:gradient_Ukappa},
\begin{align}
\Delta_m
&=
\left\langle
\nabla U_\kappa(\widehat{\mathbf d}_m),
\psi_m-\widehat\psi_m
\right\rangle
\nonumber\\
&=
\sum_{s,a}
\frac{c_{s,a}^{\kappa}}
{\widehat d_{s,a}(t_m)^{\kappa/2}}
\left(
\psi_m(s,a)-\widehat\psi_m(s,a)
\right)
\nonumber\\
&=
\sum_{s,a}
\left(
\frac{c_{s,a}}
{\sqrt{\widehat d_{s,a}(t_m)}}
\right)^\kappa
\left(
\psi_m(s,a)-\widehat\psi_m(s,a)
\right)
\nonumber\\
&\le
SA\,
c_{\max}^{\kappa}\,
\eta^{-\kappa/2}
\left\|
\psi_m(s,a)
-
\frac{\nu_m(s,a)}{\tau_m}
\right\|_\infty.
\label{eq:Delta_first_bound}
\end{align}

The last inequality follows from
\eqref{eq:d_hat_lbound}.
Furthermore, from
Lemma~\ref{lma:bound_empirical_occupancy},
with probability at least $1-\delta$,
\[
\left\|
\psi_m(s,a)
-
\frac{\nu_m(s,a)}{\tau_m}
\right\|_\infty
\le
\frac{
\sqrt{2B/\gamma_{\min}}
+
\sqrt{\ln(4SA/\delta)/2}
}{
\sqrt{\tau_m}
}
+
\frac{20B}{\gamma_{\min}\tau_m},
\]
where
\[
B
=
\log\left(
\frac{S^2A^{S+1}}{\delta/2}
\sqrt{\frac{1}{\eta}}
\right).
\]

Hence, for $m\ge m_\delta$, with probability at least $1-\delta$,
\[
\Delta_m
\le
\frac{c_1}{\sqrt{\tau_m}}
+
\frac{c_2}{\tau_m},
\]
where
\[
\begin{cases}
c_1
=
SA\,c_{\max}^{\kappa}\eta^{-\kappa/2}
\left(
\sqrt{2B/\gamma_{\min}}
+
\sqrt{\ln(4SA/\delta)/2}
\right),
\\[1.2em]
c_2
=
SA\,c_{\max}^{\kappa}\eta^{-\kappa/2}
\dfrac{20B}{\gamma_{\min}}.
\end{cases}
\]

Thus, for all $m\ge m_\delta$, with probability at least $1-\delta$,
\begin{equation}
\label{eq:Delta_bound_step2}
\Delta_m
\le
\frac{c_1}{\sqrt{\tau_m}}
+
\frac{c_2}{\tau_m}.
\end{equation}

\paragraph{Bound $\epsilon_m$:}

Recall
\[
\epsilon_m
=
\left\langle
\nabla U_\kappa(\widehat{\mathbf d}_m),
\psi^\star_m-\psi_m
\right\rangle,
\]
where
\[
\psi^\star_m
\in
\arg\max_{\mathbf d\in\mathcal D_\eta^{(P)}}
\left\langle
\nabla U_\kappa(\widehat{\mathbf d}_m),
\mathbf d
\right\rangle.
\]

For the convergence analysis, the optimistic variant of
\kexplorer uses the weights
\begin{equation}
\label{eq:wm_def}
w_m(s,a)
:=
\frac{
\left(\widehat c^+_{s,a}(t_m)\right)^\kappa
}{
\bigl(\widehat d_{s,a}(t_m)\bigr)^{\kappa/2}
},
\end{equation}
and computes $(\widetilde P_m,\psi_m)$ such that
\begin{equation}
\label{eq:algo2_oracle_step}
(\widetilde P_m,\psi_m)
\in
\arg\max_{
\tilde p\in\mathcal B_{t_m},
\ \mathbf d\in\mathcal D_\eta^{(\tilde p)}
}
\langle w_m,\mathbf d\rangle,
\end{equation}
where
\[
\mathcal B_{t_m}
:=
\prod_{(s,a)}\mathcal B_{s,a}(t_m).
\]

On the confidence event $\mathcal Q$, we have
\[
P(\cdot\mid s,a)\in\mathcal B_{s, a}(t_m),
\qquad
\forall (s,a)\in\mathcal S\times\mathcal A,
\]
and therefore
\[
P\in\mathcal B_{t_m}.
\]
Hence
\[
\psi^\star_m\in\mathcal D_\eta^{(P)}
\]
is feasible for the optimistic optimization problem
\eqref{eq:algo2_oracle_step}. By optimality of $\psi_m$,
\begin{equation}
\label{eq:optimism_innerprod}
\langle w_m,\psi_m\rangle
\ge
\langle w_m,\psi^\star_m\rangle.
\end{equation}

Adding and subtracting the optimistic linear objective yields
\begin{align}
\epsilon_m
&=
\left\langle
\nabla U_\kappa(\widehat{\mathbf d}_m),
\psi^\star_m
\right\rangle
-
\left\langle
\nabla U_\kappa(\widehat{\mathbf d}_m),
\psi_m
\right\rangle
\nonumber\\
&=
\underbrace{
\langle w_m,\psi^\star_m\rangle
-
\langle w_m,\psi_m\rangle
}_{\le 0\ \text{by }\eqref{eq:optimism_innerprod}}
\nonumber\\
&\qquad
+
\left\langle
\nabla U_\kappa(\widehat{\mathbf d}_m)-w_m,
\psi^\star_m
\right\rangle
+
\left\langle
w_m-\nabla U_\kappa(\widehat{\mathbf d}_m),
\psi_m
\right\rangle.
\label{eq:eps_decomposition}
\end{align}

Since
\[
\widehat c^+_{s,a}(t_m)
\ge
c_{s,a}
\]
on the event $\mathcal Q$, and $x\mapsto x^\kappa$ is nondecreasing
on $\mathbb R_+$ for $\kappa\ge1$, we have
\[
\left(\widehat c^+_{s,a}(t_m)\right)^\kappa
\ge
c_{s,a}^\kappa.
\]
Therefore,
\[
w_m(s,a)
-
\nabla U_\kappa(\widehat{\mathbf d}_m)(s,a)
\ge 0.
\]
Therefore,
\begin{align}
\epsilon_m
&\le
\left\langle
w_m-\nabla U_\kappa(\widehat{\mathbf d}_m),
\psi^\star_m
\right\rangle
+
\left\langle
w_m-\nabla U_\kappa(\widehat{\mathbf d}_m),
\psi_m
\right\rangle.
\label{eq:eps_reduce_two_terms}
\end{align}

For $m\ge m_\delta$, by \eqref{eq:d_hat_lbound},
\[
\widehat d_{s,a}(t_m)\ge\eta,
\qquad
\forall (s,a).
\]
Moreover,
\[
\nabla U_\kappa(\widehat{\mathbf d}_m)(s,a)
=
\frac{
c_{s,a}^{\kappa}
}{
\bigl(\widehat d_{s,a}(t_m)\bigr)^{\kappa/2}
}
=
\left(
\frac{
c_{s,a}
}{
\sqrt{\widehat d_{s,a}(t_m)}
}
\right)^\kappa.
\]
Hence, using the definition of $w_m$ in \eqref{eq:wm_def},
\begin{align}
w_m(s,a)
-
\nabla U_\kappa(\widehat{\mathbf d}_m)(s,a)
&=
\frac{
\left(\widehat c^+_{s,a}(t_m)\right)^\kappa
-
c_{s,a}^{\kappa}
}{
\bigl(\widehat d_{s,a}(t_m)\bigr)^{\kappa/2}
}
\nonumber\\
&\le
\eta^{-\kappa/2}
\left[
\left(\widehat c^+_{s,a}(t_m)\right)^\kappa
-
c_{s,a}^{\kappa}
\right].
\label{eq:wm_grad_gap}
\end{align}

Since both $\psi_m$ and $\psi^\star_m$ are distributions over
$\mathcal S\times\mathcal A$, plugging
\eqref{eq:wm_grad_gap} into
\eqref{eq:eps_reduce_two_terms} gives
\begin{equation}
\label{eq:eps_gap_sum}
\epsilon_m
\le
2\eta^{-\kappa/2}
\sum_{s,a}
\left[
\left(\widehat c^+_{s,a}(t_m)\right)^\kappa
-
c_{s,a}^{\kappa}
\right].
\end{equation}

Recall that
\[
\widehat c^+_{s,a}(t)
=
\min
\left\{
\sqrt{S-1},
\widehat c_{s,a}(t)+e_{s,a}(t)
\right\},
\]
where, for $T_{s,a}(t)\ge1$,
\[
e_{s,a}(t)
=
S
\sqrt{
\frac{
2\log(2S/\delta_t)
}{
T_{s,a}(t)
}
}.
\]
When $T_{s,a}(t)=0$, we set
$\widehat c^+_{s,a}(t)=\sqrt{S-1}$.

Let
\[
C_S:=\sqrt{S-1}.
\]
By construction and on the event $\mathcal Q$,
\[
0
\le
c_{s,a}
\le
\widehat c^+_{s,a}(t)
\le
C_S.
\]

Assume $\kappa\ge1$. By the mean value theorem, there exists
\[
\xi_{s,a}(t)
\in
[c_{s,a},\widehat c^+_{s,a}(t)]
\subseteq[0,C_S]
\]
such that
\begin{align}
\left(\widehat c^+_{s,a}(t)\right)^\kappa
-
c_{s,a}^{\kappa}
&=
\kappa
\xi_{s,a}(t)^{\kappa-1}
\left(
\widehat c^+_{s,a}(t)-c_{s,a}
\right)
\nonumber\\
&\le
\kappa
C_S^{\kappa-1}
\left(
\widehat c^+_{s,a}(t)-c_{s,a}
\right).
\label{eq:mvt_c_gap}
\end{align}

Moreover, on the same confidence event,
\[
\left|
\widehat c_{s,a}(t)-c_{s,a}
\right|
\le
e_{s,a}(t).
\]
Since
$\widehat c^+_{s,a}(t)
\le
\widehat c_{s,a}(t)+e_{s,a}(t)$,
we obtain
\[
\widehat c^+_{s,a}(t)-c_{s,a}
\le
2e_{s,a}(t).
\]
Therefore,
\begin{equation}
\label{eq:c_kappa_gap_le_e}
\left(\widehat c^+_{s,a}(t)\right)^\kappa
-
c_{s,a}^{\kappa}
\le
2\kappa
C_S^{\kappa-1}
e_{s,a}(t).
\end{equation}

Combining
\eqref{eq:eps_gap_sum}
and
\eqref{eq:c_kappa_gap_le_e}
yields
\begin{equation}
\label{eq:eps_final_step3}
\epsilon_m
\le
4\kappa
C_S^{\kappa-1}
\eta^{-\kappa/2}
\sum_{s,a}
e_{s,a}(t_m).
\end{equation}

From \eqref{eq:d_hat_lbound}, for every $(s,a)$ and
$m\ge m_\delta$,
\[
\widehat d_{s,a}(t_m)
=
\frac{
T_{s,a}(t_m)
}{
t_m
}
\ge
\eta,
\]
and hence
\[
T_{s,a}(t_m)
\ge
\eta t_m.
\]
Therefore,
\[
e_{s,a}(t_m)
\le
S
\sqrt{
\frac{
2\log(2S/\delta_{t_m})
}{
\eta t_m
}
}.
\]
Summing over all $(s,a)$ gives
\begin{equation}
\label{eq:sum_e_bound}
\sum_{s,a}
e_{s,a}(t_m)
\le
SA\cdot S
\sqrt{
\frac{
2\log(2S/\delta_{t_m})
}{
\eta t_m
}
}.
\end{equation}
Consequently,
\begin{equation}
\label{eq:epsilon_rate}
\epsilon_m
=
\widetilde O
\left(
\kappa
C_S^{\kappa-1}
\eta^{-\kappa/2}
t_m^{-1/2}
\right).
\end{equation}

\paragraph{Plugging bounds into the recursion:}

For $m\ge m_\delta$, combining the core recursion
\eqref{eq:algo2_core_recursion}
with the tracking bound
\eqref{eq:Delta_bound_step2}
and the estimation bound
\eqref{eq:eps_final_step3}
gives, on the event $\mathcal Q$,
\begin{align}
\rho_{m+1}
&\le
(1-\beta_m)\rho_m
+
C_\eta\beta_m^2
+
\beta_m
\left(
\frac{c_1}{\sqrt{\tau_m}}
+
\frac{c_2}{\tau_m}
\right) \\
& \qquad
+
4\kappa
C_S^{\kappa-1}
\eta^{-\kappa/2}
\beta_m
\sum_{s,a} e_{s,a}(t_m)
.
\label{eq:algo2_recursion_simpler}
\end{align}

Using \eqref{eq:sum_e_bound}, we have
\[
\sum_{s,a} e_{s,a}(t_m)
\le
\frac{\tilde d_1}{\sqrt{t_m}},
\]
where
\[
\tilde d_1
:=
SA\cdot S
\sqrt{
\frac{
\log(2S/\delta_{t_m})
}{
2\eta
}
}.
\]
Hence,
\begin{align}
\rho_{m+1}
&\le
(1-\beta_m)\rho_m
+
C_\eta\beta_m^2
+
\beta_m
\left(
\frac{c_1}{\sqrt{\tau_m}}
+
\frac{c_2}{\tau_m}
\right)
+
\beta_m
\frac{d_1}{\sqrt{t_m}},
\label{eq:algo2_recursion_simpler2}
\end{align}
where
\[
d_1
:=
4\kappa
C_S^{\kappa-1}
\eta^{-\kappa/2}
\tilde d_1.
\]

\smallskip
\noindent\textbf{Specializing to $\tau_m=\tau_1m^2$.}

Assume
\[
\tau_m=\tau_1m^2.
\]
Then
\[
t_m
=
\tau_1
\sum_{i=1}^{m-1}
i^2
+
1
=
\tau_1
\frac{
(m-1)m(2m-1)
}{
6
}
+
1,
\]
and
\[
\beta_m
=
\frac{
\tau_m
}{
t_{m+1}-1
}
=
\frac{
\tau_1m^2
}{
\tau_1
\frac{
m(m+1)(2m+1)
}{
6
}
}
=
\frac{
6m
}{
(m+1)(2m+1)
}.
\]
Therefore,
\[
\frac{1}{m}
\le
\beta_m
\le
\frac{3}{m},
\qquad
m\ge1.
\]

Moreover, since
\[
t_m=\Theta(m^3),
\]
there exists a constant $c_t>0$, depending only on $\tau_1$, such that
for all sufficiently large $m$,
\[
t_m
\ge
c_tm^3.
\]
Hence,
\[
\frac{1}{\sqrt{t_m}}
\le
\frac{1}{\sqrt{c_t}\,m^{3/2}}.
\]

Using these relations, for all $m\ge m_\delta$,
\begin{align}
C_\eta\beta_m^2
&\le
\frac{9C_\eta}{m^2},
\\
\beta_m\frac{c_1}{\sqrt{\tau_m}}
&\le
\frac{3}{m}
\frac{c_1}{\sqrt{\tau_1}m}
=
\frac{3c_1}{\sqrt{\tau_1}}
\frac{1}{m^2},
\\ 
\beta_m\frac{c_2}{\tau_m}
&\le
\frac{3}{m}
\frac{c_2}{\tau_1m^2}
=
\frac{3c_2}{\tau_1}
\frac{1}{m^3},
\end{align}
and
\[
\beta_m\frac{d_1}{\sqrt{t_m}}
\le
\frac{3}{m}
\frac{d_1}{\sqrt{c_t}m^{3/2}}
=
\frac{3d_1}{\sqrt{c_t}}
\frac{1}{m^{5/2}}.
\]

In particular, the last two terms are of smaller order than
$1/m^2$. Therefore, there exists a constant
\[
b_\delta
=
9C_\eta
+
\frac{3c_1}{\sqrt{\tau_1}}
+
\frac{3c_2}{\tau_1}
+
\frac{3d_1}{\sqrt{c_t}}
>0
\]
such that, for all $m\ge m_\delta$,
\begin{equation}
\label{eq:algo2_recursion_final_form}
\rho_{m+1}
\le
\left(
1-\frac{1}{m}
\right)
\rho_m
+
\frac{b_\delta}{m^2}.
\end{equation}

We now define a comparison sequence $(u_n)_{n\ge q}$.
Choose an integer $q\ge m_\delta$ such that
\[
u_q:=\max\{\rho_q,0\}.
\]
For $n\ge q$, let
\[
u_{n+1}
=
\left(
1-\frac{1}{n}
\right)
u_n
+
\frac{b_\delta}{n^2}.
\]
By induction,
\[
\rho_n
\le
u_n,
\qquad
\forall n\ge q.
\]

Multiplying the recursion by $n$ gives
\[
n u_{n+1}
=
(n-1)u_n
+
\frac{b_\delta}{n}.
\]
Thus,
\[
n u_{n+1}
-
(n-1)u_n
=
\frac{b_\delta}{n}.
\]
Telescoping from $n=q$ to $n=M-1$ yields
\[
(M-1)u_M
-
(q-1)u_q
=
b_\delta
\sum_{n=q}^{M-1}
\frac{1}{n}.
\]
Therefore,
\[
(M-1)u_M
\le
(q-1)u_q
+
b_\delta
\log\left(
\frac{M-1}{q-1}
\right)
+
b_\delta,
\]
and hence
\[
u_M
=
O\left(
\frac{\log M}{M}
\right).
\]
Since
\[
\rho_M\le u_M,
\]
we obtain
\[
\rho_M
=
\widetilde O\left(
\frac{1}{M}
\right).
\]

Finally,
\[
t_M
=
\tau_1
\frac{
(M-1)M(2M-1)
}{
6
}
+
1
=
\Theta(M^3).
\]
Therefore,
\[
M
=
\Theta(t_M^{1/3}),
\]
and consequently
\[
\rho_M
=
U_\kappa(\mathbf d_\kappa^\star)
-
U_\kappa(\widehat{\mathbf d}(t_M))
=
\widetilde O
\left(
t_M^{-1/3}
\right).
\]

This completes the proof.
\end{proof}

\section{Proof of Proposition~\ref{proposition:kappa_inf_l1}}
\label{appendix:Uk_minimax_proof}

\begin{proof}
Define
\[
r_{s,a}
:=
\frac{c_{s,a}}{\sqrt{d_{s,a}}},
\qquad
V(\mathbf d)
:=
\max_{(s,a)\in\mathcal S\times\mathcal A} r_{s,a}.
\]

By the definition of \(U_\kappa\),
\[
U_\kappa(\mathbf d)
=
\frac{1}{1-\kappa/2}
\sum_{s,a}
c_{s,a}^{\kappa}
d_{s,a}^{1-\kappa/2}.
\]
Therefore,
\begin{align*}
\left(1-\frac{\kappa}{2}\right)
U_\kappa(\mathbf d)
&=
\sum_{s,a}
c_{s,a}^{\kappa}
d_{s,a}^{1-\kappa/2} \\
&=
\sum_{s,a}
d_{s,a}
\left(
\frac{c_{s,a}}{\sqrt{d_{s,a}}}
\right)^{\kappa} \\
&=
\sum_{s,a}
d_{s,a} r_{s,a}^{\kappa}.
\end{align*}

Let \((\bar s,\bar a)\) attain the maximum in the
definition of \(V(\mathbf d)\), so that
\[
r_{\bar s,\bar a}
=
V(\mathbf d).
\]

Since \(r_{s,a}\leq V(\mathbf d)\) for every
\((s,a)\), and \(\sum_{s,a} d_{s,a}=1\), we have
\begin{align*}
\left(1-\frac{\kappa}{2}\right)
U_\kappa(\mathbf d)
&=
\sum_{s,a} d_{s,a} r_{s,a}^{\kappa} \\
&\leq
V(\mathbf d)^\kappa
\sum_{s,a}d_{s,a} \\
&=
V(\mathbf d)^\kappa.
\end{align*}
Hence,
\[
\left[
\left(1-\frac{\kappa}{2}\right)
U_\kappa(\mathbf d)
\right]^{1/\kappa}
\leq
V(\mathbf d).
\]

On the other hand, retaining only the term corresponding to
\((\bar s,\bar a)\) gives
\begin{align*}
\left(1-\frac{\kappa}{2}\right)
U_\kappa(\mathbf d)
&=
\sum_{s,a} d_{s,a}r_{s,a}^{\kappa} \\
&\geq
d_{\bar s,\bar a}
r_{\bar s,\bar a}^{\kappa} \\
&\geq
d_{\min}
V(\mathbf d)^\kappa.
\end{align*}
Therefore,
\[
d_{\min}^{1/\kappa}
V(\mathbf d)
\leq
\left[
\left(1-\frac{\kappa}{2}\right)
U_\kappa(\mathbf d)
\right]^{1/\kappa}.
\]

Combining the two bounds, we obtain
\[
d_{\min}^{1/\kappa}
V(\mathbf d)
\leq
\left[
\left(1-\frac{\kappa}{2}\right)
U_\kappa(\mathbf d)
\right]^{1/\kappa}
\leq
V(\mathbf d).
\]

Now suppose
\[
\kappa
\geq
\frac{\log(1/d_{\min})}{\log(1+\epsilon)}.
\]
Then
\[
\frac{\log(1/d_{\min})}{\kappa}
\leq
\log(1+\epsilon),
\]
and hence
\[
d_{\min}^{1/\kappa}
=
\exp\left(
-\frac{\log(1/d_{\min})}{\kappa}
\right)
\geq
\frac{1}{1+\epsilon}.
\]
Since \(\epsilon\in(0,1)\),
\[
\frac{1}{1+\epsilon}
\geq
1-\epsilon.
\]
It follows that
\begin{align} \notag
    (1-\epsilon)
V(\mathbf d)
&\leq
\left[
\left(1-\frac{\kappa}{2}\right)
U_\kappa(\mathbf d)
\right]^{1/\kappa} \\ \notag
&\leq
V(\mathbf d)
\leq
(1+\epsilon)
V(\mathbf d).
\end{align}
This proves the desired result.
\end{proof}

\section{Proof of Theorem~\ref{thm:near-minimax-transition-estimation}}
\label{appendix:proof_cor}

\begin{proof}
Define
\[
G_\kappa(\mathbf d)
:=
\left[
\left(1-\frac{\kappa}{2}\right)
U_\kappa(\mathbf d)
\right]^{1/\kappa}.
\]
Since $\kappa>2$, we have $1-\kappa/2<0$. Hence, maximizing
$U_\kappa$ over $\mathcal D$ is equivalent to minimizing
$G_\kappa$ over $\mathcal D$. In particular, if
\[
\mathbf d_\kappa^\star
\in
\arg\max_{\mathbf d\in\mathcal D}U_\kappa(\mathbf d),
\]
then
\begin{equation}
G_\kappa(\mathbf d_\kappa^\star)
\leq
G_\kappa(\mathbf d^\star).
\label{eq:Gkappa-optimality}
\end{equation}

Let
\[
\Delta_M
:=
U_\kappa(\mathbf d_\kappa^\star)
-
U_\kappa\bigl(\widehat{\mathbf d}(t_M)\bigr).
\]
By Theorem~\ref{thm:general_convergence_rate}, on its
high-probability event,
\begin{equation}
0\leq \Delta_M
=
\widetilde O\!\left(t_M^{-1/3}\right).
\label{eq:DeltaM-rate}
\end{equation}

Using the definition of $G_\kappa$, we have
\begin{align*}
G_\kappa\bigl(\widehat{\mathbf d}(t_M)\bigr)^\kappa
&=
\left(1-\frac{\kappa}{2}\right)
U_\kappa\bigl(\widehat{\mathbf d}(t_M)\bigr)
\\
&=
\left(1-\frac{\kappa}{2}\right)
\left(
U_\kappa(\mathbf d_\kappa^\star)-\Delta_M
\right)
\\
&=
G_\kappa(\mathbf d_\kappa^\star)^\kappa
+
\left(\frac{\kappa}{2}-1\right)\Delta_M.
\end{align*}
Therefore,
\begin{equation}
G_\kappa\bigl(\widehat{\mathbf d}(t_M)\bigr)
=
\left[
G_\kappa(\mathbf d_\kappa^\star)^\kappa
+
\left(\frac{\kappa}{2}-1\right)\Delta_M
\right]^{1/\kappa}.
\label{eq:Gkappa-gap-identity}
\end{equation}

The function $f(x)=x^{1/\kappa}$ is concave on $(0,\infty)$.
Hence, for every $x>0$ and $y\geq0$,
\begin{equation}
(x+y)^{1/\kappa}
\leq
x^{1/\kappa}
+
\frac{1}{\kappa}x^{1/\kappa-1}y.
\label{eq:concavity-power}
\end{equation}
Applying~\eqref{eq:concavity-power} to
\[
x=G_\kappa(\mathbf d_\kappa^\star)^\kappa,
\qquad
y=\left(\frac{\kappa}{2}-1\right)\Delta_M,
\]
gives
\begin{equation}
\begin{aligned}
G_\kappa\bigl(\widehat{\mathbf d}(t_M)\bigr)
\leq {}&
G_\kappa(\mathbf d_\kappa^\star)
\\
&+
\frac{\frac{\kappa}{2}-1}
{\kappa
 G_\kappa(\mathbf d_\kappa^\star)^{\kappa-1}}
\Delta_M.
\end{aligned}
\label{eq:Gkappa-concavity-bound}
\end{equation}
Combining~\eqref{eq:Gkappa-optimality} and
\eqref{eq:Gkappa-concavity-bound}, we obtain
\begin{equation}
\begin{aligned}
G_\kappa\bigl(\widehat{\mathbf d}(t_M)\bigr)
\leq {}&
G_\kappa(\mathbf d^\star)
\\
&+
\frac{\frac{\kappa}{2}-1}
{\kappa
 G_\kappa(\mathbf d_\kappa^\star)^{\kappa-1}}
\Delta_M.
\end{aligned}
\label{eq:Gkappa-wc-comparison}
\end{equation}

By the assumed lower bounds on
$\widehat{\mathbf d}(t_M)$ and
$\mathbf d^\star$, together with
\[
\kappa
\geq
\frac{\log(1/\underline d)}
     {\log(1+\varepsilon)},
\]
Proposition~\ref{proposition:kappa_inf_l1} applies to both
allocations. Thus,
\begin{equation}
(1-\varepsilon)
V
\bigl(\widehat{\mathbf d}(t_M)\bigr)
\leq
G_\kappa\bigl(\widehat{\mathbf d}(t_M)\bigr),
\label{eq:prop7-lower}
\end{equation}
and
\begin{equation}
G_\kappa(\mathbf d^\star)
\leq
(1+\varepsilon)
V(\mathbf d^\star).
\label{eq:prop7-upper}
\end{equation}

Substituting~\eqref{eq:prop7-upper} into
\eqref{eq:Gkappa-wc-comparison}, and then using
\eqref{eq:prop7-lower}, yields
\[
\begin{aligned}
(1-\varepsilon)
V
\bigl(\widehat{\mathbf d}(t_M)\bigr)
\leq {}&
(1+\varepsilon)
V(\mathbf d^\star)
\\
&+
\frac{\frac{\kappa}{2}-1}
{\kappa
 G_\kappa(\mathbf d_\kappa^\star)^{\kappa-1}}
\Delta_M.
\end{aligned}
\]
Dividing both sides by $1-\varepsilon$ gives
\begin{equation}
\begin{aligned}
V
\bigl(\widehat{\mathbf d}(t_M)\bigr)
\leq {}&
\frac{1+\varepsilon}{1-\varepsilon}
V(\mathbf d^\star)
\\
&+
\frac{\frac{\kappa}{2}-1}
{\kappa(1-\varepsilon)
 G_\kappa(\mathbf d_\kappa^\star)^{\kappa-1}}
\Delta_M.
\end{aligned}
\label{eq:finite-time-worst-bound}
\end{equation}

Finally, substituting~\eqref{eq:DeltaM-rate} into
\eqref{eq:finite-time-worst-bound}, and recalling that
$\kappa$ and $\varepsilon$ are fixed, we conclude that
\[
V
\bigl(\widehat{\mathbf d}(t_M)\bigr)
\leq
\frac{1+\varepsilon}{1-\varepsilon}
V(\mathbf d^\star)
+
\widetilde O\!\left(t_M^{-1/3}\right).
\]
The hidden constant is independent of $M$, although it may depend on
$\kappa$, $\varepsilon$, and the problem-dependent constants.
\end{proof}